\documentclass[12pt]{article}
\usepackage[utf8]{inputenc}
\usepackage[margin=1in]{geometry}
\usepackage{amsmath,amsthm,amssymb,amstext}
\usepackage{graphicx}
\usepackage{tikz}
\usetikzlibrary{tikzmark}
\usepackage{url}
\usepackage{enumerate}
\usepackage{float}
\usepackage{floatflt}
\usepackage{algorithm}
\usepackage{algorithmic}
\usepackage{subcaption}
\usepackage{authblk}
\usepackage{mathtools}
\DeclarePairedDelimiter{\ceil}{\lceil}{\rceil}

\newtheorem{theorem}{Theorem}
 
\newtheorem{lemma}{Lemma}

\newtheorem{definition}{Definition}

\DeclareMathOperator*{\argmax}{argmax} % no space, limits underneath in displays
\DeclareMathOperator*{\argmin}{argmin}


\makeatletter
\renewcommand\subparagraph{\@startsection{subparagraph}{5}{\parindent}%
    {3.25ex \@plus1ex \@minus .2ex}%
    {0.75ex plus 0.1ex}% space after heading
    {\normalfont\normalsize\bfseries}}
\makeatother

\begin{document}
\title{\bf Solving Multi-Objective MDP with Lexicographic Preference:\\ An application to stochastic planning with multiple quantile objective }%replace X with the appropriate number
\author[1]{Yan Li}
\author[2]{Zhaohan Sun}
\affil[1]{School of Mathematics, Georgia Institute of Technology}
\affil[2]{H. Milton Stewart School of Industrial and Systems Engineering, Georgia
Institute of Technology}

\date{} %replace with date of lecture
\maketitle

\abstract
\noindent In most common settings of Markov Decision Process (MDP), an agent evaluate a policy based on expectation of (discounted) sum of rewards. However in many applications this criterion might not be suitable from two perspective: first, in risk aversion situation expectation of accumulated rewards is not robust enough, this is the case when distribution of accumulated reward is heavily skewed; another issue is that many applications naturally take several objective into consideration when evaluating a policy, for instance in autonomous driving an agent needs to balance speed and safety when choosing appropriate decision. In this paper, we consider evaluating a policy based on a sequence of quantiles it induces on a set of target states, our idea is to reformulate the original problem into a multi-objective MDP problem with lexicographic preference naturally defined. For computation of finding an optimal policy, we proposed an algorithm \textbf{FLMDP} that could solve general multi-objective MDP with lexicographic reward preference.

\section{Introduction}

The most classical MDP problem consider maximizing a scalar reward's expectation \cite{feinberg}, however in many situation a single scalar objective is not enough to represent an agent's objective. For example, in self-autonomous driving one need to balance speed and safety \cite{wray}. A common approach is to use a weight vector and scalarization function to project the multi-objective function to single objective problem. However in practice it is hard to evaluate and analyze the projected problem since there might be many viable Pareto optimal solutions to the original problem \cite{pineda}. On the other hand, in some cases, an agent might have explicit preference over the objectives, that is, an agent might expect to optimize the higher priority objective over the lower priority ones when finding optimal policy. For example, in autonomous driving an agent would consider safety the highest priority, placing speed in the second place. \\

\noindent Several previous studies have considered such multi-objective problem with lexicographical order. Using a technique called Ordinal dynamic programming, Mitten \cite{mitten} assumed a specific preference ordering over outcomes for a finite horizon MDP; Sobel \cite{sobel} extended this model to infinite horizon MDPs. Ordinal dynamic programming has been explored under reinforcement learning. Wray et.al \cite{wray} also consider a more general setting when lexicographical order depends on initial state and slack for higher objective value is allowed for improvement over lower priority objective. In their paper they proposed an algorithm called \textbf{LVI} that tries to approximate optimal policy in infinite horizon setting, although work empirically well, the algorithm lacks theoretical guarantee, in fact, the performance could be arbitrarily worse if the MDP is adversarially designed.\\

\noindent Even in the setting that an agent indeed has only one reward, the expectation of accumulated reward is not always suitable. This is the case when the agent is risk aversion, for instance in financial market an institutional fund would like to design an auto-trading system that maximize certain lower quantile. The essential idea of such strategy is to improve the worst case situation as much as possible. Based on this motivation, Hugo and Weng \cite{Hugo} proposed quantile based reinforcement learning algorithm which seeks to optimize certain lower/upper quantile on the random outcome. In their paper they define a set of end states in finite horizon setting, let $\mathbb{P}^\pi(\cdot)$ be the probability distribution induced by policy $\pi$ on the end states, they seek to find the optimal policy in the sense that the $\tau$-lower quantile of $\mathbb{P}^\pi(\cdot)$ is maximized. Note that their objective could be improved by following observation:
\begin{enumerate}
\item Among all the policy that achieve the optimal $\tau$-lower quantile, a refined class of policy could be chosen in the sense that following such policy, the probability of ending at a state the is less preferable than the optimal quantile state is minimized.
\item Suppose $\tau_1<\tau_2$, then after finding policy class that maximize $\tau_1$-quantile, one can further find policy that maximize $\tau_2$-quantile in this policy class. For situation when multiple $\tau_i$-quantile are to be optimized, we can find optimal policy by repeating the same procedure iteratively.
\end{enumerate}

\noindent In general, if $\tau_1<\tau_2<\ldots<\tau_L$ are in consideration, we have a multi-quantile-objective MDP, in this paper, we showed a proper way to transfer this problem into a pure multi-objective MDP with lexicographic preference. To tackle computation of an optimal policy, we will introduce an algorithm called \textbf{FLMDP} that not only solve our multi-quantile-objective MDP, but also generalize multi-objective MDP with finite states, action, and horizon. Generalization to infinite states or actions to find $\epsilon$-optimal policy could be done fairly easy with small modification in our algorithm.

\section{Problem Definition}\label{probdef}
We consider finite horizon problem here, a multi-objective Markov Decision Process is described by a tuple $(S,A,P,\mathbf{R})$ where:
\begin{itemize}
\item S is finite state space.
\item A is finite action space.
\item G is finite end state space.
\item T is finite horizon.
\item P is transition function given by: $P(s,a,s^{\prime})=\mathbb{P}(s^{\prime}|s,a)$, i.e., the probability of transiting from $s$ to $s^{\prime}$ after performing action $a$.
\item $\mathbf{R}=[R_1,R_2,\ldots,R_k]$ is reward vector, with each component $R_i(s,a,s^{\prime})$ defining reward of starting from state $s$, performing action $a$ and transit to state $s^{\prime}$.
\end{itemize}
Without loss of generality we may assume $G=\{g_1,\ldots,g_n\}$. On $G$ we may define our preference as $g_1\leqslant g_2\leqslant \hdots \leqslant g_n$ where $g_i \leqslant g_j$ denotes $g_j$ is preferred over $g_i$. To enforce end state nature of set $G$, we further define transition probability and reward function have following properties:
\begin{align*}
P(g,a,g) & =1,\, \forall g \in G, \, \forall a\in A\\
R(g,a,g) & =0,\, \forall g \in G, \, \forall a\in A
\end{align*}
That is, whenever the current state is in set $G$, we remains at state $g$ until process ends at horizon T, in the meantime receiving no rewards at all. To enforce the process ends at one of the end state, we define a special end state $g_0={t=T}$ and declare $g_0 \leqslant g_i,\, \forall i\geqslant 1$.
Let $\pi$ be any policy, we define the probability distribution $\mathbb{P}^{\pi}(\cdot)$ induced by $\pi$ induced on set $G$ as $\mathbb{P}^{\pi}(g_i)=\mathbb{P}^{\pi}(s_T=g_i)$. Then we can further define cumulative distribution function.
$$F^{\pi}(g)=\sum_{g_i\leqslant g}\mathbb{P}^{\pi}(g_i)$$
The associated $\tau$-lower quantile is given by:
$$ \underline{q}_{\tau}^{\pi}=\min\{g_i:F^{\pi}(g_i)\geqslant \tau\}$$

\subsection*{Finding Optimal Policy }
Given $\tau_1<\ldots<\tau_L \in [0,1]$, following our motivation in Introduction section, our procedure to find optimal policy is a series of optimization procedure, we will show later this could be reshaped into multi-objective MDP with lexicographic preference.
\subparagraph*{Algorithmic Scheme 1}\label{algo1}
\begin{enumerate}
\item Denote $\Pi_0=\{\text{all possible policy}\}$.
\item After finding $\Pi_{i-1}$, construct $\Pi_i$:
\begin{align*}
\underline{q}_{\tau_i}^{\star} &=\max_{\pi \in \Pi_{i-1}}\{\underline{q}_{\tau_i}^{\pi}\}.\\
\hat{\Pi}_i & =\{\pi \in \Pi_{i-1}:\underline{q}_{\tau_i}^{\pi} =\underline{q}_{\tau_i}^{\star}\}
\end{align*}
That is, $\hat{\Pi}_i$ is the set of policy that maximize the $\tau_i$ quantile in $\Pi_i$. Let $p_i$ be the "biggest" state that is "smaller" than $\underline{q}_{\tau_i}^{\star}$. Here "biggest" and "smaller" should be interpreted in terms of preference. Then to minimize the probability of ending at a state that is less preferable than $\underline{q}_{\tau_i}^{\star}$, we should have $\Pi_i$ as follows:
\begin{equation} \label{eq1}
\Pi_i=\argmin_{\pi \in \hat{\Pi}_i}F^{\pi}(p_i)
\end{equation}
\item Proceed as step 2 until we have found $\Pi_{L}$. Then any policy $\pi$ that is in $\Pi_{L}$ will be our optimal policy.
\end{enumerate}
Note however it is unclear how to translate such algorithmic scheme into a tractable algorithm, the problem is that we do not know how to properly ''choose'' an policy from a policy class. We'll tackle this issue in the next section.

\section{Multi-Quantile-Objective MDP}
In this section we first present a lemma that generalize the Lemma 1 of Hugo and Weng \cite{Hugo}, this lemma fully characterize the $\underline{q}_{\tau_i}^{\star}$
\begin{lemma}\label{lemma1}
For $i=1,\ldots,L$, let $\underline{q}_{\tau_i}^{\star}$ and $\Pi_i$ be defined as before, then $\underline{q}_{\tau_i}^{\star}$ satisfies the following condition:
\begin{align*}
\underline{q}_{\tau_i}^{\star} & =\min \{g:F_i^{\star}(g)\geqslant\tau_i\}\\
F_i^{\star}(g) & =\min_{\pi \in \Pi_{i-1}} F^{\pi}(g),\, \forall g \in G
\end{align*}
\end{lemma}
\begin{proof}
We proof by induction:\\
For i=1: observe that $$F_1^{\star}(g)\leqslant F^{\pi}(g),\quad \forall \pi, \, \forall g$$
This follows directly from the definition of $F_1^{\star}(g)$. Hence the $\tau_1$-quantile of $F_1^{\star}(g)$(denoted as $g_{i_1}$) is greater or equal than $\underline{q}_{\tau_1}^{\pi}$ for all $\pi$.\\
Now by the definition of $g_{i_1}$, we have $F_1^{\star}(g_{i_1})\geqslant \tau_1$ and $F_1^{\star}(g_{i_1-1})<\tau_1$. Then by definition of $F_1^{\star}()$, we have $\exists \pi_1$, s.t.:
\begin{align*}
F^{\pi_1}(g_{i_1-1}) & =F_1^{\star}(g_{i_1-1})<\tau_1 \\
F^{\pi_1}(g_{i_1}) & \geqslant F_1^{\star}(g_{i_1})\geqslant \tau_1
\end{align*}
This means that $g_{i_1}$ is $\tau_1$-quantile of both $F_1^{\star}()$ and $F^{\pi_1}()$. Hence we have $g_{i_1}\geqslant \underline{q}_{\tau_1}^{\pi},\, \forall \pi$, and $g_{i_1}= \underline{q}_{\tau_1}^{\pi_1}$. Thus $\underline{q}_{\tau_1}^{\star}=g_{i_1}$ by definition of $\underline{q}_{\tau_1}^{\star}$.\\
Assume the claim holds for $i<k$:\\
For $i=k$: observe that $$F_k^{\star}(g)\leqslant F^{\pi}(g),\quad \forall \pi \in \Pi_{k-1}, \, \forall g$$
Hence the $\tau_k$-quantile of $F_k^{\star}(g)$(denoted as $g_{i_k}$) is greater or equal than $\underline{q}_{\tau_k}^{\pi}$ for all $\pi \in \Pi_{k-1}$.\\
Now by the definition of $g_{i_k}$, we have $F_k^{\star}(g_{i_k})\geqslant \tau_k$ and $F_k^{\star}(g_{i_k-1})<\tau_k$. Then by definition of $F_k^{\star}()$, we have $\exists \pi_k \in \Pi_{k-1}$, s.t.:
\begin{align*}
F^{\pi_k}(g_{i_k-1}) & =F_k^{\star}(g_{i_k-1})<\tau_k \\
F^{\pi_k}(g_{i_k}) & \geqslant F_k^{\star}(g_{i_k})\geqslant \tau_k
\end{align*}
This means that $g_{i_k}$ is $\tau_k$-quantile of both $F_k^{\star}()$ and $F^{\pi_k}()$. Hence we have $g_{i_k}\geqslant \underline{q}_{\tau_k}^{\pi},\, \forall \pi$, and $g_{i_k}= \underline{q}_{\tau_k}^{\pi_k}$. Thus $\underline{q}_{\tau_k}^{\star}=g_{i_k}$ by definition of $\underline{q}_{\tau_k}^{\star}$. By induction, proof complete.
\end{proof}

Following the proof of Lemma \ref{lemma1}, we could construct $\Pi_i$ as follows:
\subparagraph*{Algorithmic Scheme 2}\label{algo2}
\begin{enumerate}
\item Let $\Pi_0$=\{all possible policy\}.
\item Suppose $\Pi_{i-1}$ has been constructed, then we construct $\Pi_i$ as following:
$$
\underline{q}_{\tau_i}^{\star} =\max_{\pi \in \Pi_{i-1}}\{\underline{q}_{\tau_i}^{\pi}\}.
$$
Let $p_i$ the same as before, i.e. $p_i$ be the ''biggest state'' that is ''smaller'' than $\underline{q}_{\tau_i}^{\star}$. Then we construct $\Pi_i$ as follows:
\begin{equation}\label{eq2}
\Pi_i=\argmin_{\pi \in \Pi_{i-1}}F^{\pi}(p_i)
\end{equation}
Note that in equation (\ref{eq2})  we construct $\Pi_i$ here directly from $\Pi_{i-1}$ instead of from $\hat{\Pi}_i$ in equation (\ref{eq1}), the reason here is that by proof of Lemma \ref{lemma1}, the policy $\pi$ that minimize $F^{\pi}(p_i)$ also has $\underline{q}_{\tau_i}^{\star}$ as its $\tau_i$-quantile.
\end{enumerate}

\noindent Solving the Algorithmic Scheme mentioned before is hard in general, but giving our work before
we are now ready to formulate the previous Algorithmic Scheme into a MDP with Lexicographical objective preference. We may now restrict ourself in the setting that $\underline{q}_{\tau_1}^{\star}, \underline{q}_{\tau_2}^{\star} \cdots \underline{q}_{\tau_L}^{\star}$ are known beforehand, and consider the more general case later.\\

\noindent To do this, we define reward functions $\{R_i\}_{i=1}^L$ as follows:
\begin{equation}\label{eq4}
R_i(s_t,a_t,s_{t+1})=
\begin{cases}
1 & \text{if $s_t \not\in G$ and $s_{t+1}=g_i,\, g_i\geqslant \underline{q}_{\tau_i}^{\star}$}\\
0 & \text{otherwise}
\end{cases}
\end{equation}
Then it is easy to verify that $\mathbb{E}^{\pi}[\sum_{t=0}^T R_i(s_t,a_t,s_{t+1})]=1-F^{\pi}(p_i)$. Hence minimizing $F^{\pi}(p_i)$ is equivalent to maximizing expected reward of the MDP.\\
Define $V_i^{\pi}=\mathbb{E}^{\pi}[\sum_{t=1}^T R_i(s_t,a_t,s_{t+1})]$ the expected total reward corresponding to reward function $R_i$, then equation (\ref{eq2}) becomes as:
\begin{equation}\label{eq3}
\Pi_i=\argmax_{\pi \in \Pi_{i-1}}V_i^{\pi}
\end{equation}
We will show in the next subsection, if $\{\underline{q}_{\tau_i}^{\star}\}_{i=1}^L$ are known, the procedure described in Algorithmic Scheme 2 exactly corresponds to the procedure of solving a multi-objective MDP with lexicographic preference.

\subsection*{Multi-Objective MDP with Lexicographic Preference}

\begin{definition}\label{deflex}
Recall that a point $\bar{u}$ is lexicographical larger than 0 if $u_i$ = 0 for i =1,2 $\cdots$ j and $u_j > 0$  for some $1 \leqslant j \leqslant n$, we write $u = (u_1,u_2 \cdots u_n) >_l 0$. We then define our lexicographical order index as j, which is the first index in the vector that strictly larger than zero. Thus say $\bar{u}$ is lexicographical larger than $\bar{v}$ if  $\bar{u} - \bar{v} >_l 0$
\end{definition}

\noindent A multi-objective MDP differs from standard MDP that it has reward vector $\mathbf{R}(s,a)=[R_1(s,a),\ldots,R_L(s,a)]$ and associated value vector $\mathbf{V}(s)=[V_1(s),\ldots,V_L(s)]$, and a preference is defined on the value function associated with different rewards, say $V_1(s)>V_2(s)>\ldots>V_L(s)$. Classic multi-objective MDP seeks to find a policy that has Pareto optimal value vector. With lexicographic preference defined on value vectors, we say a policy $\pi^{\star}$ is lexicographic optimal if there is no policy $\pi$ so that $\mathbf{V}^{\pi}(s)>_l \mathbf{V}^{\pi^{\star}}(s)$. \\

\noindent In pure algorithmic scheme, an multi-objective MDP is solved by iteratively finding the optimal policy class for lower priority value function in the optimal policy class for higher priority ones. That is, denote $\Pi_0$=\{any policy\}, $\Pi_{i+1}$ is found by:
$$\Pi_{i+1}=\argmax_{\pi \in \Pi_i} V_{i+1}^{\pi}(s)$$

\noindent In our multi-quantile-objective MDP, $(S,A,P,\mathbf{R})$ is defined as the same as in section \ref{probdef}, the reward functions $R_i$ is defined as in equation (\ref{eq4}). With value vector $\mathbf{V}^{\pi}=(V_1^{\pi},\ldots,V_L^{\pi})$, we define lexicographical preference on $\mathbf{V}^{\pi}$ as defined in definition \ref{deflex}. Then with equation (\ref{eq3}) replacing equation (\ref{eq2}) in Algorithmic 2, it is easy to see that Algorithmic Scheme 2 now become a procedure of solving multi-objective lexicographic MDP with
parameters $(S,A,P,\mathbf{R},\mathbf{V})$.

\section{Solving Multi-Quantile-Objective MDP}
Solving Multi-Quantile-Objective MDP lies in general situation of solving multi-objective MDP with lexicographical preference. A natural one is to shape the original problem to a sequence of constrained case MDP and solve this sequence of constrained MDP iteratively. In the next subsection we proposed an algorithm that can solve general multi-objective MDP with lexicographic preference directly, thus solving multi-quantile-objective MDP here is just a special case.
\subsection*{Constrained MDP formulation}
The following procedure reshape a multi-objective MDP with lexicographic preference to a sequence of constrained MDP problem.
\begin{enumerate}
\item At step $1$, $\Pi_0$=\{all possible policy\}. Optimize objective $V_1^{\pi}$, $V_1^{\star}=\max_{\pi \in \Pi_0} V_1^{\pi}$.
\item At step $i$, Optimize objective $V_i^{\pi}$ with constraints:
\begin{equation*}
\begin{aligned}
& \underset{\pi}{\text{minimize}}
& & V_i^{\pi} \\
& \text{subject to}
& & V_j^{\pi} \geqslant V_j^{\star}, \; j = 1, \ldots, i-1.
\end{aligned}
\end{equation*}
\item Proceed as in 2 until step L is finished.
\end{enumerate}
It is easy to see that at step $i$ the constraints in the optimization procedure naturally restrict the algorithm to search policy in the class that is identical to $\Pi_{i-1}$, thus correctness of this reshape is guaranteed. Altman \cite{altman} has shown that an optimal randomized policy could be found in such constrained MDP, Chen and Feinberg \cite{Chen} also showed how to find optimal deterministic policy. Note this type of algorithm indeed does unnecessary work by restarting from searching whole policy space in every step. In this next subsection, we design a dynamic programming flavor algorithm that finds an optimal deterministic policy for general lexicographic order MDP.

\subsection*{Lexicographic Markov Decision Process}
In this subsection we introduce an algorithm \textbf{FLMDP} that solves general lexicographic MDP in finite horizon, in particular it can be used to solve our previous formulated multi-quantile-objective MDP. \\

\noindent Let $V_{i,t}^{\pi}$: $L \times S \times T \rightarrow R$ be the expected reward obtained by using policy $\pi$ in decision epochs t, t+1, $\cdots$ T, here, for simplicity, we let reward of end state equals zero, thus $V_{i,t}^{\pi}$ can be represented as

$$V_{i,t}^{\pi}(s) = \mathbb{E}_{s_t=s}^{\pi}[\sum_{n=t}^T R_i(s_n,a_n)]$$
Note that although in our problem $R_i()$ relates to out next state, we can solve this problem by simply define $R_i(s_t,a_t)=\mathbb{E}[R_i(s_t,a_t,s_{t+1})]$ with expectation taken w.r.t $s_{t+1}$.\\

\noindent We first define state value function:
$$ Q_{i,t}^{\pi}(s,a) = R_i(s,a) + \sum_{s' \in S} Pr(s'|s,a) V_{i,t+1}^{\pi}(s')$$
Then following the definition in constrained MDP, $\forall t = T, T-1 \cdots 1$, and $\forall i = 1, 2 \cdots L$, we define restricted bellman equation operator $B_i^t$ as
$$ B_i^tV_{i,t}^{\pi}(s) = \max_{a \in A_{i-1}^t}\{R_i(s,a) + \sum_{s' \in S} Pr(s'|s,a) V_{i,t+1}^{\pi}(s')\} $$
where
$$ A_{i+1}^t(s) = \{a \in A_i^t(s)|\max_{a^\prime \in A_i^t(s)}Q_{i,t}^{\pi}(s,a^\prime) = Q_{i,t}^{\pi}(s,a)\}$$
and
$$ A_0^t(s)  = A(s) $$

{\centering
\begin{minipage}{1.0\linewidth}
\begin{algorithm}[H]
\caption{Finite-horizon Lexicographic MDP - \textbf{FLMDP}}
\label{alg:A}
\begin{algorithmic}
\STATE{Input $R_i(s,a)$, $i = 1,2 \cdots L$}
\STATE{Set $V_{i,T}^{\pi}(s) = 0$, $\forall i = 1,2 \cdots L$, $\forall s \in S$}
\FOR{$t = T-1,T-2 \cdots 1 $}
\FOR{$i = 1,2 \cdots L$}
\STATE{$ V_{i,t}^{\pi}(s) = B_i^t V_{i,t}^{\pi}(s) $}

\ENDFOR
\STATE{$\pi_t^{\star} \in A_L^t$}
\ENDFOR
\STATE{Output $\pi_1^{\star},\pi_2^{\star}, \cdots ,\pi_T^{\star}$}
\end{algorithmic}
\end{algorithm}
\end{minipage}
\par
}
\hfill  \\

\begin{theorem}
In our algorithm 1, $\forall t = T-1, T-2, \cdots 1$, $\{\pi_t^{\star}\}_{t \leqslant T-1}$ are optimal policy for our Lexicographic MDP problem.
\end{theorem}

\begin{proof}
Before beginning our proof, we need some notations. Recall:
\begin{align*}
V_{i,t}^{\pi}(s) &= \mathbb{E}_{s_t=s}^{\pi}[\sum_{t}^T R_i(s_t,a_t)]\\
V_{i,t}^{\star}(s) & =\mathbb{E}_{s_t=s}^{\pi_{\star}}[\sum_{t}^T R_i(s_t,a_t)]\\
\mathbf{V}_t^{\pi}(s) & = [V_{1,t}^{\pi}(s), \ldots, V_{L,t}^{\pi}(s)]\\
\mathbf{V}_t^{\star}(s) &= [V_{1,t}^{\star}(s),\ldots, V_{L,t}^{\star}(s)]
\end{align*}
where $\{\pi_t^{\star}\}$ denotes the policy output by Algorithm 1. Then $V_{j,t}^{\pi}(s)$ defines the value function associated with reward $R_i$, starting a tail problem with initial state $s$ at time $t$ following given policy $\pi$. Note that $\mathbf{V}_t^{\pi}(s)$ is exactly the value vector for full horizon MDP with initial state $s$. By our specification of reward function $R_i$, we naturally have $\mathbf{V}_T^{\pi}(s)=\mathbf{0}$ and $\mathbf{V}_T^{\star}(s)=\mathbf{0}$. \\

\noindent Let $\leqslant_l,<_l,>_l,\geqslant_l$ denotes lexicographical order relationship on value vector $\mathbf{V}_t^{\pi}(s)$. We use backward induction to show that for $\forall \, \pi$, and for $\forall \, t=1,\ldots,T-1$, for $\forall \, s$, we have $\mathbf{V}_t^{\pi}(s) \leqslant_l \mathbf{V}_t^{\star}(s)$. \\

\noindent For $t=T-1$, $\mathbf{V}_{T-1}^{\pi}(s) \leqslant_l \mathbf{V}_{T-1}^{\star}(s)$ is trivial by procedure of our algorithm. A simple induction on $i$ suffice to give a formal proof, we omit the details here. \\

\noindent Suppose the claim holds for $t+1,\ldots,T-1$, now we proceed to prove the claim holds for $t$: assume $\mathbf{V}_{t+1}^{\pi}(s)<_l \mathbf{V}_{t+1}^{\star}(s)$
\begin{align*}
V_{i,t+1}^{\pi}(s) & =V_{i,t+1}^{\star}(s),\qquad i=1,\ldots,i_{t+1}-1\\
V_{i_{t+1},t+1}^{\pi}(s) & <_l V_{i_{t+1},t+1}^{\star}(s)
\end{align*}
We next show that $\mathbf{V}_{t}^{\pi}(s)<_l \mathbf{V}_{t}^{\star}(s)$ also holds:\\
\begin{enumerate}
\item if $V_{1,t}^{\pi}(s)<V_{1,t}^{\star}(s)$, then we are done.
\item if $V_{1,t}^{\pi}(s)=V_{1,t}^{\star}(s)$, construction of $\pi^{\star}$ and value iteration for finite horizon MDP gives us:
\begin{align*}
V_{1,t}^{\pi}(s) & =R_1(s,\pi(s))+\sum_{j}P(s,\pi(s),j)V_{1,t+1}^{\pi}(j)\\
V_{1,t}^{\star}(s) &= \max_{a} R_1(s,a)+\sum_{j}P(s,a,j)V_{1,t+1}^{\star}(j)
\end{align*}
By induction hypothesis we have $V_{1,t+1}^{\pi}(j) =V_{1,t+1}^{\star}(j)$, then we must have $V_{1,t}^{\pi}(s)\leqslant V_{1,t}^{\star}(s)$. Now since we have equality achieved, by our definition of $A_1(s)$ in our algorithm, we must have $\pi(s) \in A_1(s)$.
\item We now use induction to show that for if $i<i_{t+1}-1$, and
$$V_{j,t}^{\pi}(s)=V_{j,t}^{\star}(s),\, j=1,\ldots,i$$
then we must have $\pi(s) \in A_{i}(s)$ and $V_{i+1,t}^{\pi}(s)\leq V_{i+1,t}^{\star}(s)$. The base case i=1 have been proved in step 2. Suppose the claim holds for $i-1$, then for $i$:\\
By induction hypothesis we have $\pi(s) \in A_{i-1}(s)$. Construction of $\pi^{\star}$ and value iteration for finite horizon MDP gives us:
\begin{align}\label{vi}
V_{i,t}^{\pi}(s) & =R_i(s,\pi(s))+\sum_{j}P(s,\pi(s),j)V_{i,t+1}^{\pi}(j)\\
V_{i,t}^{\star}(s) &= \max_{a \in A_{i-1}(s)} R_i(s,a)+\sum_{j}P(s,a,j)V_{i,t+1}^{\star}(j)
\end{align}
By induction hypothesis we have $V_{i,t+1}^{\pi}(j) =V_{i,t+1}^{\star}(j)$, then we must have $V_{i,t}^{\pi}(s)\leqslant V_{i,t}^{\star}(s)$. Now since we have equality achieved, by our definition of $A_i(s)$ in our algorithm, we must have $\pi(s) \in A_i(s)$. Then replacing $i$ equation (\ref{vi}) with $i+1$ we have:
\begin{align*}
V_{i+1,t}^{\pi}(s) & =R_{i+1}(s,\pi(s))+\sum_{j}P(s,\pi(s),j)V_{i+1,t+1}^{\pi}(j)\\
V_{i+1,t}^{\star}(s) &= \max_{a \in A_{i}(s)} R_{i+1}(s,a)+\sum_{j}P(s,a,j)V_{i+1,t+1}^{\pi}(j)
\end{align*}
By induction hypothesis we have $V_{i+1,t+1}^{\pi}(j)= V_{i+1,t+1}^{\star}(j)$, noticing that now $\pi(s) \in A_{i}(s)$, then we have:
$$V_{i+1,t}^{\pi}(s)\leqslant V_{i+1,t}^{\star}(s)$$
Finally, when $i=i_{t+1}-1$, if
$$V_{j,t}^{\pi}(s)=V_{j,t}^{\star}(s),\, j=1,\ldots,i$$
Then following the argument as before, and utilizing that now $V_{i+1,t+1}^{\pi}(j)< V_{i+1,t+1}^{\star}(j)$, we must have:
$$V_{i+1,t}^{\pi}(s)< V_{i+1,t}^{\star}(s)$$
which gives us $\mathbf{V}_{t}^{\pi}(s)<_l \mathbf{V}_{t}^{\star}(s)$
\end{enumerate}
Notice that our previous argument could also be used to prove $\mathbf{V}_{t+1}^{\pi}(s)= \mathbf{V}_{t+1}^{\star}(s) \Rightarrow \mathbf{V}_{t}^{\pi}(s)\leqslant_l \mathbf{V}_{t}^{\star}(s)$. Then combining all the ingredients we have, the following statement holds:
\begin{equation}\label{inductive}
\mathbf{V}_{t+1}^{\pi}(s)\leqslant_l \mathbf{V}_{t+1}^{\star}(s) \Rightarrow \mathbf{V}_{t}^{\pi}(s)\leqslant_l \mathbf{V}_{t}^{\star}(s)
\end{equation}
To conclude our proof, notice we have $\mathbf{V}_{T-1}^{\pi}(s) \leqslant_l \mathbf{V}_{T-1}^{\star}(s)$, apply equation (\ref{inductive}) iteratively, we have $\mathbf{V}_{1}^{\pi}(s) \leqslant_l \mathbf{V}_{1}^{\star}(s)$, the optimality of our output policy follows immediately.
\end{proof}

\noindent Now we return to the general case where optimal quantiles $\{\underline{q}_{\tau_i}^{\star}\}_{i=1}^{L}$ is not known before hand. Out idea is to proceed iteratively, at $k$th iteration, we used bisection to guess the location of the unknown $\underline{q}_{\tau_k}^{\star}$, we then solve a lexicographic MDP with $k$ reward $[R_1,\ldots,R_k]$ and preference aligns with our preference for total $L$ reward. Specifically, at $k$-th iteration, we maintain $u$ and $l$ such that $F^{\star}_k(g_{l-1})< \tau_k$ and $F^{\star}_k(g_{u-1})\geqslant \tau_k$. We successively reduce $u-l$ by half until $u-l=1$. Then $\underline{q}_{\tau_k}^{\star}=g_{u-1}$. \\
% With $u-l=1$, if $F^{\star}_k(g_l)\geqslant\tau_k$ we have $\underline{q}_{\tau_k}^{\star}=g_l$, otherwise we have $\underline{q}_{\tau_k}^{\star}=g_u
\noindent To proceed the $k$-th iteration, we need to define our reward function as follows:
\begin{eqnarray*}
R_i^{q_{\tau_i}}(s_t,a_t,s_{t+1})=
\begin{cases}
1 & \text{if $s_t \not\in G$ and $s_{t+1}=g_i,\, g_i\geqslant q_{\tau_i}$}\\
0 & \text{otherwise}
\end{cases}
\end{eqnarray*}
The reward vector at $k$-th iteration is then given by:
$$\mathbf{R}=[R_1^{\underline{q}_{\tau_1}^{\star}},\ldots,R_{k-1}^{\underline{q}_{\tau_{k-1}}^{\star}},R_k^{\underline{q}_{\tau_k}}]$$
where $\underline{q}_{\tau_k}$ is our guess for $\underline{q}_{\tau_k}^{\star}$.
% \noindent we also implement bisection algorithm to increase the rate of converge. Then our parameter $\theta_i$ will converge to relative optimal quantile index $\theta_{i}^\star$. \\

{\centering
\begin{minipage}{1.0\linewidth}
\begin{algorithm}[H]
\caption{Multi-Quantile-Objective(MQO) MDP}
\label{alg:A}
\begin{algorithmic}
\STATE{Set $V_{i,T}^{\pi}(s) = 0$, $\forall i = 1,2 \cdots L$, $\forall s \in S$}
\FOR{$i = 1,2 \cdots L$}
\STATE{Guess a proper $\underline{q}_{\tau_i}$, which should be larger than $\underline{q}_{\tau_k}^{\star}$, $\forall k = i-1 \cdots 1$}
\STATE{Set $l$ be the largest index of $\{g_k\}$ s.t. $g_k < \underline{q}_{\tau_i}$, Set $u \leftarrow n$}
\REPEAT
\STATE{Solve Lexicographic MDP with $\underline{q}_{\tau_1}^{\star},\underline{q}_{\tau_2}^{\star} \cdots \underline{q}_{\tau_j}^{\star}$, $j \leqslant i-1$}
% \FOR{$t = T-1,T-2 \cdots 1 $}
% \FOR{$j = 1 \cdots i$}
% \STATE{Set $R_i(s,a,s^\prime) = R_i^{\theta_i}(s,a,s^\prime)$}
% \STATE{$ V_{j,t}^{\pi}(s) = B_j^t V_{j,t}^{\pi}(s) $}
% \ENDFOR
% \ENDFOR
\STATE{Output $V_{i,t}^{\pi}(s)$}
\IF{$V_{i,t}^{\pi}(s) \leqslant 1- \tau_i$}
\STATE{$l \leftarrow i$}
\ELSE
\STATE{$u \leftarrow i$}
\ENDIF
\STATE{$ i \leftarrow \ceil{\frac{l+u}{2}}$}
\UNTIL{$u-l = 1$}
\STATE{$ \underline{q}_{\tau_i}^{\star} \leftarrow g_{u-1}$}
\ENDFOR
\end{algorithmic}
\end{algorithm}
\end{minipage}
\par
}

\section{Conclusion}
In this paper we consider a multi-quantile-objective MDP problem that combines previous work in quantile objective MDP and multi-objective MDP. Our contribution is two folds, first we formulate the problem into multi-objective MDP problem, the second is that our algorithm to solve this problem could also solve general multi-objective MDP problem with finite horizon, state space and action space. Extension to infinite state space or action space could be also done with slight modification. \\

\noindent We note our possible future work here: Pineda et.al \cite{pineda} has showed that constrained MDP could be reshaped into a sequence of multi-objective MDP with lexicographic preference and additional slack variables, thus if one could solve lexicographic MDP with slack variable efficiently, then solution of constrained MDP follows. For finite horizon, we believe similar dynamic programming flavor algorithm could be invented for solving lexicographic MDP with slack variables, we leave it here as an open problem and our future work.


\begin{thebibliography}{9}
\bibitem{pineda}
  Luis Pineda, Kyle H. Wray, and Shlomo Zilberstein.
 Revisiting Multi-Objective MDPs with relaxed Lexicographic Preferences.
  \emph{AAAI Fall Symposium on Sequential Decision Making for
Intelligent Agents},
  pages 63–68,Arlington, Virginia, USA,
  November 2015.
\bibitem{wray}
  Kyle H. Wray, Shlomo Zilberstein, and Abdel-Illah Mouaddib.
  Multi-Objective MDPs with
Conditional Lexicographic Reward Preferences.
  \emph{In Proceedings of the Twenty-Ninth Conference
on Artificial Intelligence (AAAI)}, pages 3418–3424, Austin, TX, USA, January 2015.
\bibitem{feinberg}
  Feinberg, Eugene A., and Adam Shwartz, eds.
  \emph{Handbook of Markov decision processes: methods and applications}.
  Vol. 40. Springer Science $\&$ Business Media,
  2012.
\bibitem{Chen}
  Richard C. Chen, and Eugene A. Feinberg.
   Non-Randomized Policies for Constrained Markov Decision Processes.
  \emph{Mathematical Methods of Operations Research.}
\bibitem{mitten}
 L. G. Mitten.
 Preference Order Dynamic Programming.
 \emph{Management Science.}
 Volume 21, Issue 1, Pages 43-46, September 1974.
\bibitem{altman}
 Eitan Altman.
 Constrained Markov Decision Process.
 Chapman and Hall/CRC, 1999.
\bibitem{sobel}
 Matthew J. Sobel.
 Ordinal Dynamic Programming.
 \emph{Management Science}.
 Vol. 21, No. 9, Theory Series, pp. 967-975, May, 1975.

\bibitem{Hugo}
  Hugo Gilbert, Paul Weng.
  Quantile Reinforcement Learning.
  In
  \emph{ACM} 2016
\bibitem{Sutton}
  Sutton, Richard S., and Andrew G. Barto.
  \emph{Reinforcement learning: An introduction.}
  Vol. 1. No. 1. Cambridge: MIT press, 1998.
\bibitem{P.Weng}
  Weng, Paul.
  Markov Decision Processes with Ordinal Rewards: Reference Point-Based Preferences. In
  \emph{ICAPS.} 2011.
\end{thebibliography}
\end{document}